\documentclass[letterpaper, 10 pt, conference]{ieeeconf}

\IEEEoverridecommandlockouts
\usepackage[dvipsnames]{xcolor}

\usepackage{amsmath} %
\usepackage{amssymb}
\usepackage{url}            %
\usepackage{cite}
\let\citep\cite
\usepackage{booktabs}       %
\usepackage{multirow}
\usepackage{amsfonts}       %
\usepackage{nicefrac}       %
\usepackage{microtype}      %
\let\labelindent\relax
\usepackage{enumitem}       %
\usepackage{silence}
\usepackage{subcaption}
\usepackage{cuted}          %
\usepackage{siunitx}

\usepackage[normalem]{ulem}
\useunder{\uline}{\ul}{}

\usepackage{mathtools}      %
\definecolor{myBlue}{HTML}{156082}
\definecolor{myOrange}{HTML}{C04F15}

\usepackage{graphicx}
\usepackage{wrapfig}       %
\usepackage{stackengine}

\usepackage{amsthm}
\usepackage{amsmath}
\usepackage{amssymb}   %
\usepackage{amsfonts}
\usepackage{mathtools}
\usepackage[export]{adjustbox}

\usepackage[percent]{overpic}
\usepackage{array}
\usepackage[hidelinks]{hyperref}

\theoremstyle{plain}

\theoremstyle{definition}

\theoremstyle{remark}

\makeatletter
\newcommand{\labelcurrentequation}[1]{%
  \begingroup
  \protected@edef\@currentlabel{\theequation}%
  \label{#1}%
  \endgroup
}

\makeatother

\definecolor{AuroraTeal}{HTML}{0A9396} %
\colorlet{AuroraTealLight}{AuroraTeal!35}
\colorlet{AuroraTealDark}{AuroraTeal!70!black}

\newcommand{\method}{\texttt{MS-PPO}}

\newif\ificraanonymous
\icraanonymousfalse

\title{\LARGE \bf
Beyond Topology: A Morphological Symmetry Graph Representation for Locomotion Policy Learning
}

\ificraanonymous
\author{Anonymous Authors%
}
\hypersetup{pdfauthor={Anonymous Authors}}
\else
\author{Sizhe Wei$^{1*}$, Xulin Chen$^{2*}$, Fengze Xie$^{3}$, Garrett Ethan Katz$^{2}$, Zhenyu Gan$^{2}$, and Lu Gan$^{1\dagger}$%
\thanks{$^{1}$Sizhe Wei and Lu Gan are with the Georgia Institute of Technology. {\tt\small \{swei,lgan\}@gatech.edu}}%
\thanks{$^{2}$Xulin Chen, Garrett Ethan Katz, and Zhenyu Gan are with Syracuse University. {\tt\small \{xchen168,gkatz01,zgan02\}@syr.edu}}%
\thanks{$^{3}$Fengze Xie is with the California Institute of Technology. {\tt\small fxxie@caltech.edu}}%
\thanks{$^{*}$Equal contribution. $^{\dagger}$Corresponding author.}%
}
\hypersetup{pdfauthor={Sizhe Wei, Xulin Chen, Fengze Xie, Garrett Ethan Katz, Zhenyu Gan, Lu Gan}}
\fi

\begin{document}
\maketitle
\thispagestyle{empty}
\pagestyle{empty}

\begingroup
\setlength{\stripsep}{4pt}
\begin{strip}
    \ificraanonymous\else\par\vspace{-18mm}\fi
    \centering
    \includegraphics[width=0.9\textwidth]{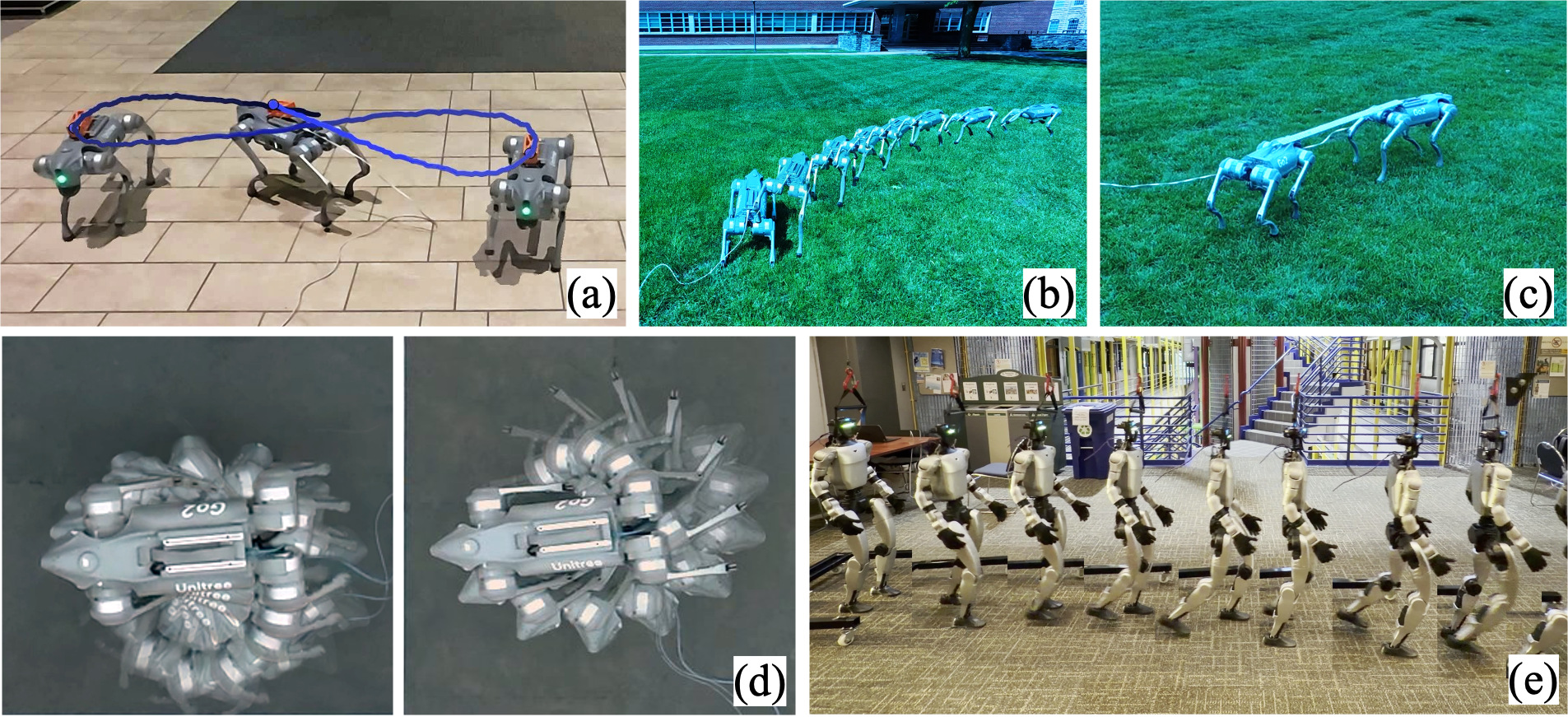}
    \captionof{figure}{We deploy policies based on our representation on Unitree Go2 and Unitree G1 robots:
    \textbf{(a)} Figure-8 command tracking with real-robot trajectories
    (Sec.~\ref{sec:exp-walk-figure-8}); \textbf{(b)}--\textbf{(c)} stable
    locomotion and collaboration under rear-right calf joint failure
    (Sec.~\ref{sec:exp-joing-failure}); \textbf{(d)} symmetry generalization
    to unseen mirrored commands
    (Sec.~\ref{sec:exp-mirrored-commands}); and \textbf{(e)} 
    zero-shot sim-to-real transfer to G1 humanoid locomotion
    (Sec.~\ref{sec:exp-humanoid}). Project website: \href{https://msppo.github.io}{\texttt{msppo.github.io}}.}
    \label{fig:teaser}
\end{strip}
\endgroup

\begin{abstract}

Reinforcement learning has enabled impressive locomotion skills on articulated robots, but common policy representations remain only weakly aligned with robot physics. Generic networks ignore kinematic structure, while graph-based policies encode connectivity without specifying how physical quantities transform across symmetric body parts. We introduce a morphological symmetry graph representation for locomotion policy learning, instantiated as \method{} with PPO on Go2 and \texttt{MS-SAC} with FastSAC on G1. Starting from the robot's topological graph, our representation explicitly incorporates the permutations and sign changes induced by morphological symmetry on observations and actions, yielding a graph actor with an equivariant action mean and an invariant critic. We evaluate the representation on Unitree Go2 command tracking and asymmetric joint-failure tasks, and on Unitree G1 29-DoF whole-body locomotion, including zero-shot sim-to-real deployment. Our approach achieves the lowest tracking errors in both simulation and hardware, and the highest mean simulation returns across all three joint-failure types on Go2 among the evaluated baselines. On Go2, \method{} uses only 34.6\% of the parameters of the PPO baseline. These results support the effectiveness of jointly encoding topology and morphological symmetry across robot embodiments and learning frameworks.
\end{abstract}

\section{Introduction}

Reinforcement learning (RL)~\citep{sutton2018rl} has recently delivered strong performance in controlling multi-link rigid-body systems such as humanoids~\citep{ze2025twist, radosavovic2024real}, quadrupeds~\citep{miki2022learning}, and dexterous hands~\citep{qi2023hand}. Much of this success has been driven by high-fidelity simulators and large-scale parallel training~\citep{pmlr-v164-rudin22a}. However, most systems still use relatively unstructured policy representations, with observations flattened into vectors and mapped to actions by generic neural networks. For articulated robots, this approach does not specify the system's physical structure, but instead leaves the policy to implicitly discover such information that can be encoded through architectural priors. Our question is how to develop policy representations that explicitly reflect the robot's physics and morphology.

A topology-informed graph is a natural starting point. It explicitly represents the robot's joints and rigid links, allowing information to flow along kinematic structure within the policy network. Graph-based policies use this structure to share parameters across repeated limbs and improve generalization across embodiments~\citep{wang2018nervenet,whitman2023learning,butterfield2024mihgnn, sferrazza2025body,gupta2022metamorph,patel2025get}. However, graph connectivity alone does not define how morphological symmetries transform observations and actions. For example, it does not specify how coordinates transform under reflections, rotations, or their combinations, or how the observation and action map between symmetric joints. Parameter sharing based on topology therefore does not guarantee that mirrored observations result in correspondingly transformed actions.

Existing symmetry-aware RL methods encourage symmetric behavior through mirrored data augmentation~\citep{mittal2024symmetry}, reward shaping~\citep{ding2024towards}, or globally equivariant architectures~\citep{su2024symmloco,nie2025coordinated,li2025morphologically}. Data augmentation and reward shaping encourage symmetry through training objectives or data, while equivariant architectures enforce the corresponding transformations directly. These symmetry constraints alone do not specify how information should propagate along the robot's kinematic structure. This motivates jointly encoding topology-aware information flow and morphology-aware coordinate transformations.

In this work, we introduce a morphological symmetry graph representation for locomotion policy learning. We represent an articulated robot as a graph and explicitly define the permutations and sign changes induced by sagittal reflection on observations, privileged observations, and actions~\citep{apraez2025morphological,xie2025_mshgnn}, yielding a graph actor with an equivariant action mean and an invariant critic. We instantiate this representation as \method{} with PPO on Unitree Go2 quadruped and \texttt{MS-SAC} with FastSAC on Unitree G1 humanoid, extending evaluation to a 29-DoF whole-body action space. The evaluated tasks include symmetric command tracking, locomotion under asymmetric joint failure on Go2, and whole-body locomotion on G1. Our approach outperforms the evaluated baselines by achieving the lowest tracking errors both in simulation and on hardware. On Go2, it also achieves the highest reward for all three joint-failure types with only 34.6\% of the parameters of the PPO baseline. Actor--critic ablations show that the equivariant graph actor provides the larger performance gain, while the invariant critic further improves reward and tracking. Together, these results support the effectiveness of jointly encoding kinematic topology and morphological symmetry across robot embodiments and actor--critic frameworks.

\section{Related Work}

\subsection{Topology-Informed Robot Policy}

Topology-informed policies encode the robot's kinematic structure in the policy architecture or training process, improving sample efficiency~\cite{sferrazza2025body}, control stability and adaptability~\cite{yu2023multi}, and cross-embodiment generalization~\cite{gupta2022metamorph}. A common solution represents the robot as a graph of joints and links, enabling message passing along kinematic connections and parameter sharing across structurally similar body parts. These designs support \emph{modular} and \emph{compositional} policies that generalize across morphologies by recombining learned modules~\cite{huang2020one, whitman2023learning}. This perspective also applies naturally to Transformer-based policies, where self-attention can be interpreted as message passing over a fully connected graph~\cite{vaswani2017attention, bronstein2021geometricdeeplearninggrids}. Existing methods differ primarily in how they incorporate robot structure into the policy. Body Transformer~\cite{sferrazza2025body} masks attention using the adjacency matrix of the robot's topology. MetaMorph~\cite{gupta2022metamorph} combines morphological and proprioceptive features in tokens ordered by a depth-first traversal of the kinematic tree, with learned positional embeddings. EAT~\cite{yu2023multi} demonstrates robust real-world locomotion under online changes in morphology, and related methods~\cite{hong2021structure, patel2025get} represent joints as tokens and introduce structural biases based on kinematic connectivity. Our work goes beyond connectivity by incorporating morphological symmetry as an additional architecture prior. We use a graph neural network to encode robot topology~\cite{butterfield2024mihgnn} and augment the graph with a morphology-consistent symmetry representation~\cite{xie2025_mshgnn}. The resulting graph provides a representation space that aligns the robot's physical structure and captures both kinematic connectivity and embodiment symmetry.

\subsection{Symmetric Legged Locomotion}

Symmetry plays a central role in biological locomotion, supporting balanced, efficient, and regular gaits. In learning-based frameworks, incorporating symmetry improves sample efficiency and accelerates learning~\cite{abdolhosseini2019learning}. Prior work has incorporated symmetry via designing hybrid dynamics with odd-even symmetry~\cite{altendorfer2004stability, razavi2017symmetry}, augmenting training data~\cite{mittal2024symmetry}, shaping rewards~\cite{ding2024towards}, and using symmetry-preserving neural architectures~\cite{su2024symmloco, nie2025coordinated}. Data augmentation and reward shaping encourage symmetric behavior through soft constraints and yield approximate equivariance. Symmetry-preserving neural architectures impose stricter constraints on the policy but often require task-specific definitions of state and action symmetries. Existing methods therefore do not necessarily provide a morphology-consistent representation of how a robot's state and action coordinates transform under its embodiment symmetries. We address this gap by integrating symmetry representations directly into the topology-informed policy graph, aligning the policy with both the robot's kinematic connectivity and its morphological symmetries.

\section{Methodology}

We propose a morphological symmetry graph representation that combines kinematic connectivity with the coordinate transformations induced by robot symmetry. Figure~\ref{fig:framework} summarizes the construction: an encoder maps physical observations to graph features, an equivariant backbone propagates information along the robot topology, and task-specific readouts produce an equivariant action mean or an invariant value estimate. We instantiate the same symmetry construction into two learning frameworks: PPO on Go2 and FastSAC on G1.

\begin{figure*}[t!]
    \centering
    \includegraphics[width=0.9\linewidth]{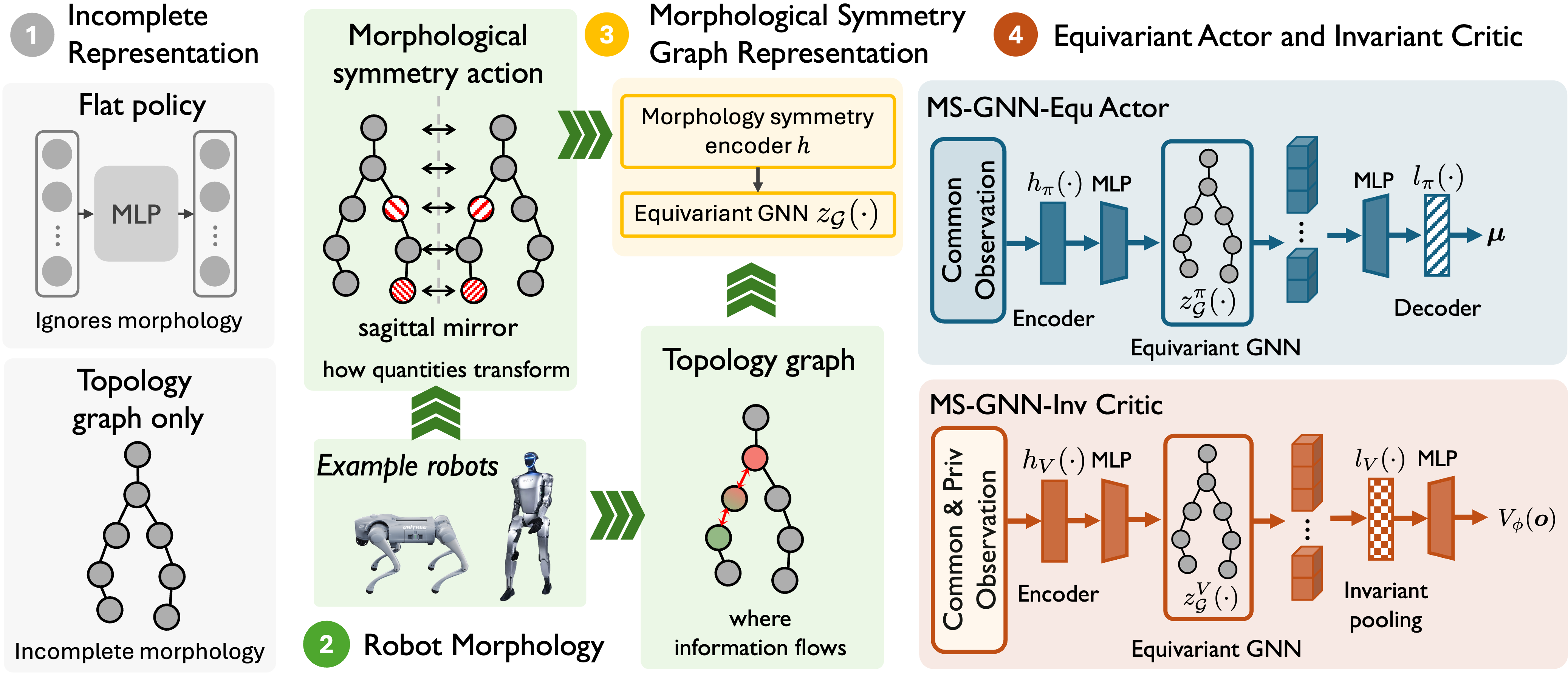}
    \caption{Overview of the morphological symmetry graph representation. Kinematic connectivity specifies where information flows, while morphology-induced permutations and sign changes specify how signals transform. The \textcolor{myBlue}{MS-GNN-Equ} actor maps reflected observations to reflected action means; the \textcolor{orange}{MS-GNN-Inv} critic produces the same value for reflected inputs.}
    \vspace{-4mm}
    \label{fig:framework}
\end{figure*}

\subsection{Symmetry Actions and MDP Assumptions}\label{sec:symmetric-mdp}
We consider sagittal reflection $\mathbb{C}_2$ group $G=\{e,m\}$, where $e$ is the identity and $m$ denotes sagittal reflection. For each vector space $\mathcal Y$, we denote its group representation by $\rho_{\mathcal Y}: G \rightarrow \mathrm{GL}(\mathcal Y)$, such that $\rho_{\mathcal Y}(gh)=\rho_{\mathcal Y}(g)\rho_{\mathcal Y}(h)$ for arbitrary $g,h\in G$. In particular, $\rho_{\mathcal O}$ and $\rho_{\mathcal A}$ act on observations and actions through morphology-induced permutations and sign changes. 
Since the full MDP state need not be a vector space, we denote its symmetry transformation more generally by $T_g:\mathcal S\rightarrow\mathcal S$. 

We formulate the MDP using an augmented state that includes the physical robot state together with task commands, persistent fault variables, and randomized environment parameters. Consequently, the symmetry transformation $T_g$ acts jointly on all symmetry-dependent components of the task. For every $g\in G$ and measurable $B\subseteq\mathcal S$, we assume
\begin{equation}
\begin{aligned}
 p_0(T_gB)&=p_0(B),\\
 r(T_gs,\rho_{\mathcal A}(g)a)&=r(s,a),\\
 P(T_gB\mid T_gs,\rho_{\mathcal A}(g)a)&=P(B\mid s,a),
\end{aligned}
\label{eq:symmetric-mdp-assumptions}
\end{equation}
where $p_0$, $r$, and $P$ denote the initial-state distribution, reward, and transition kernel. We assume that disturbance and randomization distributions are invariant under the corresponding reflection such that observation generation and termination conditions respect this symmetry. These assumptions concern distributions rather than individual realizations: uniformly sampled mirrored faults can preserve symmetry even though a fixed unilateral fault is asymmetric. Under these assumptions, the optimal value function and the value function of a distributionally equivariant policy are invariant, motivating our architectural priors.

\subsection{Robot-to-Graph Encoding}
We construct a graph $\mathcal G=(\mathcal V,\mathcal E)$ from the kinematic tree, with nodes $\mathcal{V}$ representing body components or joint groups, and edges $\mathcal{E}$ representing the connectivity. Actor inputs $o\in\mathcal O$ comprise the available proprioceptive observations and commands, with temporal histories where used. Group actions apply framewise to histories; finite history is not assumed to make the latent state fully observable.

Following the morphological symmetry graph construction of~\cite{xie2025_mshgnn}, the encoder incorporates morphology-induced permutations and sign changes so that encoded features transform solely by graph reindexing. Paired body parts use coordinate sign masks, while base features are lifted to two virtual base nodes exchanged by reflection. Given a base feature vector $b \in \mathcal{B}$ and the reflection representation $\rho_{\mathcal B}$ of the base-feature space $\mathcal{B}$, the pair $(b,\rho_{\mathcal B}(m)b)$ swaps under reflection because $\rho_{\mathcal B}(m)^2=I$, preserving information in components that change sign. This avoids discarding reflection-odd quantities at body parts that are fixed by the node permutation, such as the robot base. Sign masks follow the reflection law of each physical quantity: polar vectors (e.g., linear velocity, gravity) and axial vectors (e.g., angular velocity) use distinct masks, while joint-coordinate signs follow the mirrored joint frames.

The actor encoder $h_\pi$ and decoder $l_\pi$ satisfy
\begin{equation}
\begin{aligned}
 h_\pi(\rho_{\mathcal O}(g)o)&=P_gh_\pi(o),\\
 l_\pi(P_gZ)&=\rho_{\mathcal A}(g)l_\pi(Z).
\end{aligned}
\label{eq:h_def}
\end{equation}
\labelcurrentequation{eq:l_def}
Node-type projections and $L$ sum-aggregation graph layers share parameters consistently across symmetry orbits, yielding an equivariant graph neural network $z_{\mathcal{G}}$ on a feature space $X$
\begin{equation}
 z_{\mathcal G}(P_gX)=P_gz_{\mathcal G}(X),
 \label{eq:graph-equivariance}
\end{equation}
where the node permutation $P_g$ preserves connectivity and node/edge types on the graph $\mathcal{G}$. Changing the embodiment changes the graph, feature assignments, and coordinate transformations while preserving this construction.

\subsection{Equivariant Actor and Invariant Critic}
We use separately parameterized graph backbones $z_{\mathcal G}^{\pi}$ and $z_{\mathcal G}^{V}$ for the actor and critic, respectively. The action mean is the composition $\mu_\theta=l_\pi\circ z_{\mathcal G}^{\pi}\circ h_\pi$. Equations~\eqref{eq:h_def}--\eqref{eq:graph-equivariance} give
\begin{equation}
 \mu_\theta(\rho_{\mathcal O}(g)o)=\rho_{\mathcal A}(g)\mu_\theta(o).
 \label{eq:method-policy-equivariance}
\end{equation}
Thus, deterministic deployment maps reflected observations to reflected actions by construction.

For the PPO critic, let $u=[o,o_{\mathrm{priv}}]\in\mathcal U$ include privileged observations, with their action $\rho_{\mathcal U}$. A separate encoder and backbone obey the same equivariance identities. We define $V_\phi(u)=l_V(z_{\mathcal G}^{V}(h_V(u)))$, using an invariant readout $l_V$. 
Let $Z=z_{\mathcal G}^{V}(h_V(u))$, and let $m(i)$ denote the node paired with node $i$ under reflection.
For each reflected feature pair $(Z_i,Z_{m(i)})$, the readout concatenates $Z_i+Z_{m(i)}$, $|Z_i-Z_{m(i)}|$, and $(Z_i-Z_{m(i)})^2$ in a fixed orbit order, followed by an MLP. These statistics are unchanged by swapping the pair, so
\begin{equation}
 V_\phi(\rho_{\mathcal U}(g)u)=V_\phi(u).
 \label{eq:method-value-invariance}
\end{equation}
This enforces a symmetry prior on the critic; it does not assert that the value of every exploratory policy is invariant.

\subsection{PPO and FastSAC Instantiations}
For Go2, \method{} uses PPO~\cite{schulman2017proximal} with
\begin{equation}
 \pi_\theta(\cdot\mid o)=\mathcal N\!\left(\mu_\theta(o),\operatorname{diag}(\sigma^2)\right).
\end{equation}
The standard deviations are learned independently. Consequently, the action mean is exactly equivariant, but the training policy distribution is not constrained to be equivariant: this would additionally require $\Sigma(\rho_{\mathcal O}(g)o)=\rho_{\mathcal A}(g)\Sigma(o)\rho_{\mathcal A}(g)^\top$. Deployment uses the deterministic action mean and retains its equivariance.

For G1, we use FastSAC~\cite{seo2025learning} to accelerate training, retaining the equivariant actor design. The critic estimates an action value and is constrained under joint observation-action transformation,
\begin{equation}
 Q_\phi(\rho_{\mathcal U}(g)u,\rho_{\mathcal A}(g)a)=Q_\phi(u,a).
\end{equation}
Thus, the symmetry principles extend to both learning frameworks, with algorithm-specific value functions and training objectives.

\section{Experiments}

We evaluate our representation along three axes: symmetric command tracking, robustness to asymmetric joint failures, and scalability across robot embodiments. Experiments use two platforms with distinct morphologies, the Unitree Go2 quadruped~\cite{unitreego2} and the Unitree G1 humanoid robot~\cite{unitreeg1}. All policies are trained on NVIDIA RTX 5090 and L40S GPUs.

We compare MLP, graph, and Transformer backbones to examine the roles of topology and morphological symmetry. PPO uses generic MLP actor-critic networks~\cite{schulman2017proximal}, and PPO-Aug uses mirrored data augmentation~\cite{mittal2024symmetry}. PPO-EMLP uses a $\mathbb C_2$-equivariant mean network and an invariant critic~\cite{su2024symmloco,apraez2025morphological,cesa2022program}. MI-PPO uses a morphology-informed GNN~\cite{butterfield2024mihgnn,wang2018nervenet}; its unsigned leg-permutation symmetry does not implement the signed sagittal-reflection action. MI-PPO-Aug adds $\mathbb C_2$ data augmentation to this backbone without removing its architectural permutation constraint. We use a four-layer Body Transformer (BoT)~\cite{sferrazza2025body}, which incorporates topology through body-structured attention.

\textbf{Evaluation protocol.} The Go2 single-joint-failure comparison trains all seven methods using five common seeds (1, 4, 23, 42, 123).
We define eligibility rate (ER) as the fraction of training seeds whose final policies retain valid tracking samples after a 50-step warm-up in the no-fault \textit{walk-figure-8} test. A seed is excluded when all 16 parallel environments terminate their first episode before contributing any post-warm-up samples; early termination is triggered by abnormal base contact. The same eligible checkpoints are used for Figure-8 and joint-failure evaluation. ER reports the fraction of checkpoints passing this screen, while task metrics summarize performance over eligible seeds.

\textbf{Metrics.} 
Linear tracking error (Lin. TE) is the mean Euclidean norm of the planar (x and y) velocity error; angular tracking error (Ang. TE) is the mean absolute yaw-rate error. 
Reward is the undiscounted first-episode return, including warm-up and the terminal-step reward. 

\begin{table*}[tb]
\centering
\caption{Go2 simulation results. Topo. indicates an explicit topology prior; Symm. identifies the encoded or augmented symmetry. Task metrics use eligible seeds: $^\dagger$two for PPO-EMLP, $^\ddagger$four for MI-PPO-Aug, and five for each other method. Entries report mean $\pm$ population standard deviation across eligible seeds for Figure-8 and across eligible seed-leg pairs for each joint-failure type. Figure-8 tracking errors exclude the 50-step warm-up; rewards include it. 
}
\label{tab:numerical-results}
\resizebox{\textwidth}{!}{%
\begin{tabular}{@{}lcccccccccc@{}}
\toprule
\multirow{2}{*}{Method} & \multirow{2}{*}{ER $\uparrow$} & \multirow{2}{*}{Topo.} & \multirow{2}{*}{Symm.} & \multicolumn{3}{c}{\textit{walk-figure-8}} & \multicolumn{3}{c}{\textit{joint zero-torque}} & \multirow{2}{*}{Num. Params.$\downarrow$} \\ \cmidrule(lr){5-7}\cmidrule(lr){8-10}
& & & & Lin. TE (m/s) $\downarrow$ & Ang. TE (rad/s) $\downarrow$ & Reward $\uparrow$ & Hip reward $\uparrow$ & Thigh reward $\uparrow$ & Calf reward $\uparrow$ & \\
\midrule
PPO~\cite{schulman2017proximal} & 5/5 &  &  & $0.084 \pm 0.015$ & $0.150 \pm 0.007$ & $13.295 \pm 1.946$ & $4.797 \pm 2.158$ & $10.244 \pm 1.634$ & $2.217 \pm 1.884$ & $2,113,945$ \\
PPO-Aug~\cite{mittal2024symmetry} & 5/5 &  & $\mathbf{C}_2$-Aug & $0.077 \pm 0.020$ & $0.146 \pm 0.014$ & $14.464 \pm 0.483$ & $7.913 \pm 2.150$ & $11.869 \pm 1.741$ & $2.907 \pm 2.428$ & $2,113,945$ \\
PPO-EMLP~\cite{su2024symmloco}$^\dagger$ & 2/5 &  & $\mathbf{C}_2$ & $0.086 \pm 0.006$ & $0.130 \pm 0.008$ & $14.206 \pm 0.179$ & $4.599 \pm 3.134$ & $10.763 \pm 0.703$ & $2.332 \pm 2.024$ & $1,056,979$ \\
BoT~\cite{sferrazza2025body} & 5/5 & $\checkmark$ &  & $0.170 \pm 0.029$ & $0.200 \pm 0.030$ & $11.910 \pm 0.670$ & $0.738 \pm 1.040$ & $1.827 \pm 2.071$ & $1.048 \pm 1.030$ & $2,041,765$ \\
MI-PPO~\cite{butterfield2024mihgnn,wang2018nervenet} & 5/5 & $\checkmark$ & $\mathbf{S}_4$ & $0.089 \pm 0.014$ & $0.165 \pm 0.007$ & $14.442 \pm 0.341$ & $5.929 \pm 2.821$ & $10.802 \pm 2.087$ & $2.291 \pm 1.984$ & $864,142$ \\
MI-PPO-Aug$^\ddagger$ & 4/5 & $\checkmark$ & $\mathbf{S}_4$ ($\mathbf{C}_2$-Aug) & $0.293 \pm 0.170$ & $0.603 \pm 0.266$ & $6.945 \pm 5.135$ & $2.578 \pm 4.311$ & $3.597 \pm 5.835$ & $1.627 \pm 2.421$ & $864,142$ \\
\method{} (Ours) & 5/5 & $\checkmark$ & $\mathbf{C}_2$ & $\mathbf{0.061 \pm 0.014}$ & $\mathbf{0.101 \pm 0.011}$ & $\mathbf{14.679 \pm 1.736}$ & $\mathbf{9.418 \pm 1.606}$ & $\mathbf{12.719 \pm 1.456}$ & $\mathbf{4.023 \pm 3.045}$ & $\mathbf{732,046}$ \\
\bottomrule
\end{tabular}}
\end{table*}

\subsection{Quadrupedal Figure-8 Walking}\label{sec:exp-walk-figure-8}

For Go2 robot, the graph follows the robot's kinematic connectivity. Projected gravity and velocity commands are assigned to the base node, while joint positions, velocities, previous actions, and phase features are assigned to the corresponding joint nodes. The critic uses the same graph augmented with privileged information, such as friction and restitution coefficients, which are invariant under reflection.

Our Go2 training environment builds on \textit{Walk-These-Ways}~\cite{margolis2022walktheseways} with zero-torque actuator failures~\cite{bao2025toward}. At each episode reset, we sample nominal operation with probability $0.7$ or a single-joint failure with probability $0.3$. In a faulted episode, one of the 12 joints is selected uniformly and its commanded torque is set to zero throughout the episode. All methods use this same nominal/single-fault mixture. The actor does not receive the fault identity; all methods receive 30 frames of proprioceptive observations and previous actions, corresponding to around $0.6$ s at 50 Hz. This history provides temporal evidence of the latent actuator condition, without assuming that finite history guarantees fault identification.

The \textit{walk-figure-8} task evaluates velocity tracking through alternating left and right turns. Simulation and hardware use velocity commands derived from the same periodic reference trajectory:$x(t)=2\sin(2\pi t/20), y(t)=\sin(4\pi t/20)$,
with positions in meters and time in seconds. The reference has a 20-s period and a $4\,\mathrm{m}\times2\,\mathrm{m}$ peak-to-peak envelope; the realized path depends on velocity tracking and survival time. Simulation evaluation covers one period, while hardware evaluation uses three consecutive periods (60 s) per method. Table~\ref{tab:numerical-results} shows that \method{} achieves the lowest mean linear and angular tracking errors and the highest mean Figure-8 return. On hardware, \method{} also achieves the lowest mean linear and angular tracking errors (Table~\ref{tab:hardware_velocity_errors}); Figure~\ref{exp-result:walk-figure-8} visualizes the recorded trajectories.

Compared with the strongest baseline for each metric, \method{} reduces simulation linear and angular tracking errors by 20.8\% and 22.3\%, respectively. Achieving the lowest errors in both metrics supports coordinated translational and rotational tracking during alternating turns. This advantage extends to hardware, where the corresponding reductions are 2.8\% and 9.5\%. These results are consistent with the proposed combination of kinematic connectivity and morphological symmetry, which organizes inter-joint coordination and relates control responses across mirrored turning commands.

\begin{figure*}[tb!]
    \centering
    \includegraphics[width=0.98\linewidth]{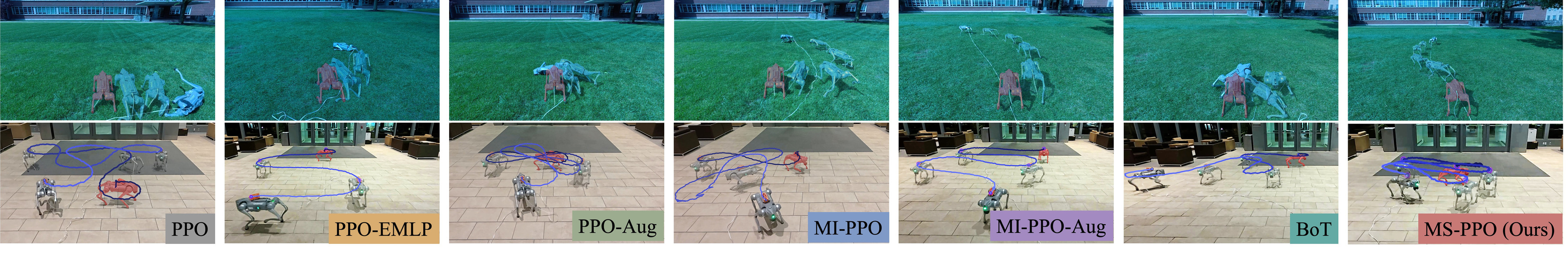}
    \captionof{figure}{
    Visualized hardware results. The top row shows \textit{joint zero-torque}, where the robot is commanded to track $v_x^{\text{cmd}}=1\text{m/s}$ with the rear-right calf joint disabled. The bottom row shows \textit{walk-figure-8}, with each method evaluated over three consecutive figure-8 cycles. Blue curves show the robot trajectories, and red indicates the initial positions. MS-PPO maintains a more stable path in \textit{joint zero-torque} and achieves better symmetric command tracking in \textit{walk-figure-8} compared to baselines. Quantitative errors are reported in Table~\ref{tab:hardware_velocity_errors}.}
    \label{exp-result:walk-figure-8}
    \label{fig:zero-torque-real-tracking}
    \vspace{-8pt}
\end{figure*}

\begin{table}[t]
    \centering
    \caption{
         The mean $\pm$ standard deviation of the tracking errors on Go2 hardware. For \textit{walk-figure-8}, errors are averaged over all three cycles. For \textit{joint zero-torque}, each method is evaluated in two trials and we report the errors using the trial with the longer survival time.
        }
    \label{tab:hardware_velocity_errors}
    \resizebox{\columnwidth}{!}{%
    \centering
    \begin{tabular}{@{}lcccc@{}}
        \toprule
        \multirow{2}{*}{Method} & \multicolumn{2}{c}{\textit{walk-figure-8}}                                                                              & \multicolumn{2}{c}{\textit{joint zero-torque}}                                                     \\ \cmidrule(l){2-3} \cmidrule(l){4-5}
                                & Lin. TE (m/s) $\downarrow$ & Ang. TE (rad/s) $\downarrow$ & Lin. TE (m/s) $\downarrow$ & Ang. TE (rad/s) $\downarrow$ \\ \midrule
        PPO
        & $0.222 \pm 0.116$ & $0.332 \pm 0.240$
        & $1.203 \pm 0.247$ & $0.486 \pm 0.378$ \\
         PPO-Aug
        & $0.217 \pm 0.102$ & $0.266 \pm 0.146$
        & $1.209 \pm 0.183$ & $0.892 \pm 0.927$ \\
        PPO-EMLP
        & $0.301 \pm 0.143$ & $0.291 \pm 0.182$
        & $1.673 \pm 1.810$ & $0.714 \pm 0.553$ \\
       BoT
        & $0.273 \pm 0.141$ & $0.333 \pm 0.225$
        & $1.241 \pm 0.222$ & $0.900 \pm 0.473$ \\
        MI-PPO
        & $0.239 \pm 0.122$ & $0.308 \pm 0.177$
        & $0.825 \pm 0.287$ & $0.580 \pm 0.550$ \\
        MI-PPO-Aug
        & $0.256 \pm 0.115$ & $0.243 \pm 0.167$
        & $0.543 \pm 0.337$ & $0.538 \pm 0.473$ \\
        \method{} (Ours)
        & $\mathbf{0.211 \pm 0.112}$ & $\mathbf{0.220 \pm 0.125}$
        & $\mathbf{0.405 \pm 0.263}$ & $\mathbf{0.246 \pm 0.216}$ \\ \bottomrule
            \end{tabular}
    }
    \vspace{-5mm }
\end{table}

\textbf{Generalization to mirrored commands.}\label{sec:exp-mirrored-commands}
In a separate experiment, we train Go2 trotting and pronking policies using only nonnegative lateral and yaw-rate commands: $c_x\in[-1,1]$ m/s, $c_y\in[0,0.6]$ m/s, and $c_\omega\in[0,1]$ rad/s. We evaluate the same policies on trained and unseen mirrored commands, which reverse $c_y$ and $c_\omega$ while preserving $c_x$. In simulation, \method{} maintains comparable tracking errors across the two command sets for both gaits. On hardware, the trotting policy performs 15-s in-place turns at $c_x=c_y=0$ and $c_\omega=\pm1$ rad/s, sustaining stable turning in both the trained and unseen directions without retraining. These results demonstrate generalization to commands related by morphological reflection. See the accompanying video for hardware demonstrations.

\subsection{Joint Failure Tolerance}\label{sec:exp-joing-failure}

The \textit{joint zero-torque} experiment evaluates compensation for an unobserved localized actuator failure. We command forward walking at 1\,m/s and evaluate all 12 single-joint failures for up to 1000 steps, using 100 parallel environments per seed and fault condition. Figure~\ref{fig:zero-torque-bar-plot} reports individual faults; Table~\ref{tab:numerical-results} aggregates first-episode returns across the four legs for each joint type. \method{} achieves the highest mean hip-, thigh-, and calf-failure returns among the seven methods.

\textbf{Symmetry and local compensation.} Reflection maps a unilateral fault to its counterpart on the opposite side. Uniform sampling of these mirrored faults preserves the symmetry of the fault family when dynamics and rewards transform consistently. Equivariant policies therefore produce side-specific compensation during faulted rollouts while relating their responses across mirrored observation histories. The 30-frame history provides temporal evidence of the unobserved actuator condition, and graph connectivity organizes its propagation between mechanically connected joints.

\textbf{Training success and representation.} All five \method{} checkpoints pass the eligibility screen, compared with two for PPO-EMLP and four for MI-PPO-Aug. Both PPO-EMLP and \method{} impose the same morphological equivariance on the policy mean; \method{} additionally encodes kinematic connections and structured parameter sharing. Despite using fewer parameters than PPO-EMLP (732,046 versus 1,056,979), \method{} achieves higher ER and rewards (Table~\ref{tab:numerical-results}). These results support topology and structured sharing as useful inductive biases for learning efficient fault compensation.

MI-PPO treats the four legs as interchangeable under unsigned $\mathbf{S}_4$ permutations, overlooking their distinct spatial roles and the coordinate sign changes required by physical reflection. This constraint extends beyond the robot's morphological symmetry and can restrict position-dependent compensation for a single-joint failure. MI-PPO-Aug pairs each original training sample with its signed reflection, doubling the training samples while retaining the same unsigned permutation constraint. The augmented samples therefore do not correct the architectural mismatch and may make optimization more difficult, consistent with its lower ER and fault returns relative to MI-PPO. In contrast, \method{} encodes the signed morphological reflection together with kinematic connectivity to support coordinated, position-dependent compensation.

\textbf{Hardware fault tolerance.} Figure~\ref{fig:zero-torque-real-tracking} and Table~\ref{tab:hardware_velocity_errors} evaluate rear-right calf failure on Go2. To limit the risk of hardware damage, we conduct two trials per method and select the longer-surviving trial for both visualization and quantitative evaluation. In these selected trials, \method{} exhibits stable and sustained forward locomotion and achieves the lowest mean linear and angular tracking errors. Relative to the strongest baseline for each metric, it reduces linear error by 25.4\% and angular error by 49.4\%. These improvements in translational and rotational tracking are consistent with the visually steadier forward motion under joint failure.

Figure~\ref{fig:zero-torque-collaboration} further shows dual-robot locomotion with a faulted lead robot. Interaction with the rear robot introduces additional disturbances through dragging or perturbing the lead robot. \method{} maintains stable locomotion in the recorded sequence despite the combined actuator failure and interaction disturbances, providing qualitative evidence of robustness beyond isolated single-robot failures.

\begin{figure}[t]
    \centering
    \includegraphics[width=\linewidth]{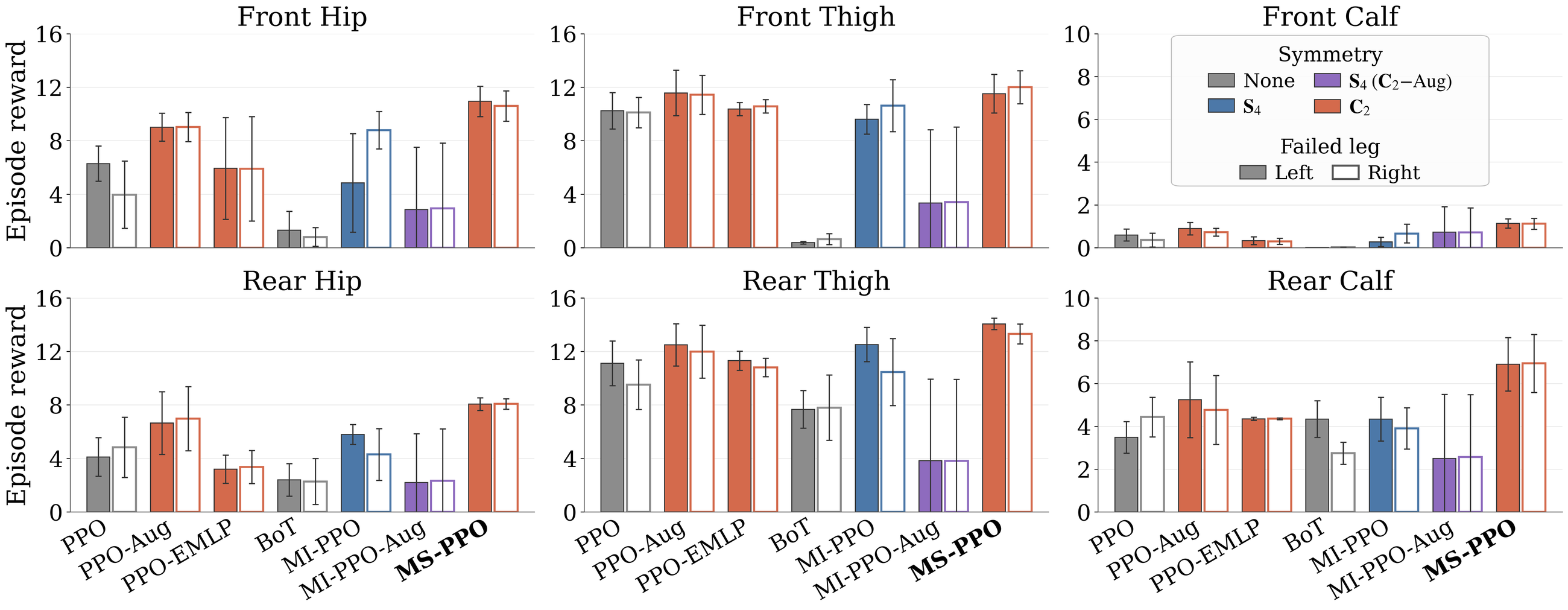}
    \caption{Go2 joint-failure tolerance in simulation. Each group reports first-episode return for one disabled joint. Bars show means and standard deviations across eligible seeds.}
    \label{fig:zero-torque-bar-plot}
    \vspace{-5mm}
\end{figure}

\subsection{Humanoid Whole-Body Locomotion}\label{sec:exp-humanoid}

We evaluate the applicability of our representation across robot embodiments and its scalability to higher-dimensional control on Unitree G1 in \textit{Holosoma}~\cite{Amazon_FAR_and_Abbeel_Holosoma, Mittal_Isaac_Lab_-_2025}. The 29-DoF whole-body action space comprises 12 leg, 3 waist, and 14 arm joints, requiring coordinated upper- and lower-body motion during velocity tracking. We instantiate the representation on the G1 morphology and train policies specifically for this task.

\textbf{Training framework.} To accelerate training, we use FastSAC~\cite{seo2025learning} for all seven humanoid methods. Our method retains the representation design used in the Go2 experiments, with an equivariant graph actor and an invariant critic adapted to the SAC formulation. We denote it by \texttt{MS-SAC} and rename the PPO-based baselines accordingly. All methods are trained on the same 29-DoF task within the FastSAC framework.

\textbf{Whole-body graph and symmetry constraints.} The actor receives a 100-dimensional observation containing joint positions, velocities, previous actions, base angular velocity, projected gravity, velocity commands, and gait phases. We organize local joint features into a 14-node graph: two symmetry-related virtual base nodes, three nodes per leg (hip, knee, ankle), and three nodes per arm (shoulder, elbow, wrist). Waist features are encoded at the two virtual base nodes, which represent two symmetry-related feature slots rather than two physical bases. Corresponding left and right body parts share encoders and decoders, with morphology-induced sign changes and permutations connecting physical coordinates to graph features. The waist decoder combines the paired base embeddings according to the reflection parity of its three actions. This construction extends the topology and signed reflection representation to all 29 action coordinates.

\textbf{Simulation and hardware evaluation.} Table~\ref{tab:g1-sim-locomotion} compares all seven methods in simulation and on the physical G1. Simulation evaluation uses eight velocity-command groups, 80 parallel environments, and 1000 evaluation steps. We aggregate command-group metrics within each seed and report mean $\pm$ standard deviation over five training seeds. Simulation reward is the mean per-step reward over the evaluation horizon. Given the practical constraints of hardware testing, we evaluate a single forward-velocity command of $0.7\,\mathrm{m/s}$ on G1. Deterministic actors are deployed without fine-tuning in the same 29-DoF action space. Hardware results report mean $\pm$ standard deviation across three trials per method. Since the hardware logs do not directly record base linear velocity, we reconstruct planar base velocity offline from joint-encoder-based foot kinematics and IMU measurements under a stance-foot assumption. Yaw rate is obtained directly from the IMU gyroscope.

\texttt{MS-SAC} achieves the highest mean simulation reward and the lowest mean linear and angular tracking errors, supporting the representation's scalability to coordinated 29-DoF whole-body control. Its advantage extends to hardware, where it attains the lowest mean errors in both metrics. Relative to SAC-EMLP, the strongest hardware baseline, \texttt{MS-SAC} reduces linear error from $0.170$ to $0.146\,\mathrm{m/s}$ (approximately 14\%), while achieving a slightly lower angular error ($0.139$ versus $0.140\,\mathrm{rad/s}$). SAC-EMLP also outperforms the MLP-based SAC baseline in both hardware metrics, consistent with the value of explicit morphological symmetry for humanoid locomotion. The additional linear-tracking improvement of \texttt{MS-SAC}, together with its advantage over MI-SAC, supports combining kinematic graph structure with physically consistent symmetry constraints. Together with the Go2 results, these findings support applicability across robot embodiments and effective zero-shot sim-to-real tracking under the evaluated hardware command.

\begin{table}[tb]
\centering
\caption{G1 29-DoF whole-body locomotion. Simulation results report mean $\pm$ std over five seeds and eight velocity-command groups; hardware results use a $0.7\,\mathrm{m/s}$ forward command without fine-tuning.}
\label{tab:g1-sim-locomotion}
\label{tab:g1-real-locomotion}
\resizebox{\columnwidth}{!}{%
\begin{tabular}{@{}lccccc@{}}
\toprule
\multirow{2}{*}{Method}
& \multicolumn{3}{c}{Simulation}
& \multicolumn{2}{c}{Hardware} \\
\cmidrule(lr){2-4}\cmidrule(l){5-6}
& Reward $\uparrow$
& Lin. TE ($\mathrm{m/s}$) $\downarrow$
& Ang. TE ($\mathrm{rad/s}$) $\downarrow$
& Lin. TE ($\mathrm{m/s}$) $\downarrow$
& Ang. TE ($\mathrm{rad/s}$) $\downarrow$ \\
\midrule
SAC~\cite{seo2025learning}
& $0.154 \pm 0.002$
& $0.395 \pm 0.022$
& $0.139 \pm 0.009$
& $0.185 \pm 0.019$
& $0.158 \pm 0.011$ \\
SAC-Aug~\cite{mittal2024symmetry}
& $0.156 \pm 0.001$
& $0.374 \pm 0.022$
& $0.138 \pm 0.009$
& $0.191 \pm 0.005$
& $0.188 \pm 0.017$ \\
SAC-EMLP~\cite{su2024symmloco}
& $0.153 \pm 0.003$
& $0.403 \pm 0.041$
& $0.136 \pm 0.016$
& $0.170 \pm 0.004$
& $0.140 \pm 0.024$ \\
SAC-BoT~\cite{sferrazza2025body}
& $0.147 \pm 0.010$
& $0.475 \pm 0.117$
& $0.254 \pm 0.158$
& $0.289 \pm 0.015$
& $0.154 \pm 0.034$ \\
MI-SAC~\cite{butterfield2024mihgnn,wang2018nervenet}
& $0.153 \pm 0.002$
& $0.414 \pm 0.046$
& $0.147 \pm 0.016$
& $0.215 \pm 0.017$
& $0.258 \pm 0.061$ \\
MI-SAC-Aug
& $0.152 \pm 0.003$
& $0.410 \pm 0.045$
& $0.156 \pm 0.019$
& $0.188 \pm 0.017$
& $0.222 \pm 0.033$ \\
\texttt{MS-SAC} (ours)
& $\mathbf{0.159 \pm 0.004}$
& $\mathbf{0.338 \pm 0.037}$
& $\mathbf{0.123 \pm 0.023}$
& $\mathbf{0.146 \pm 0.009}$
& $\mathbf{0.139 \pm 0.007}$ \\
\bottomrule
\end{tabular}}
\end{table}

\begin{figure}[tb]
    \centering
    \includegraphics[width=\linewidth]{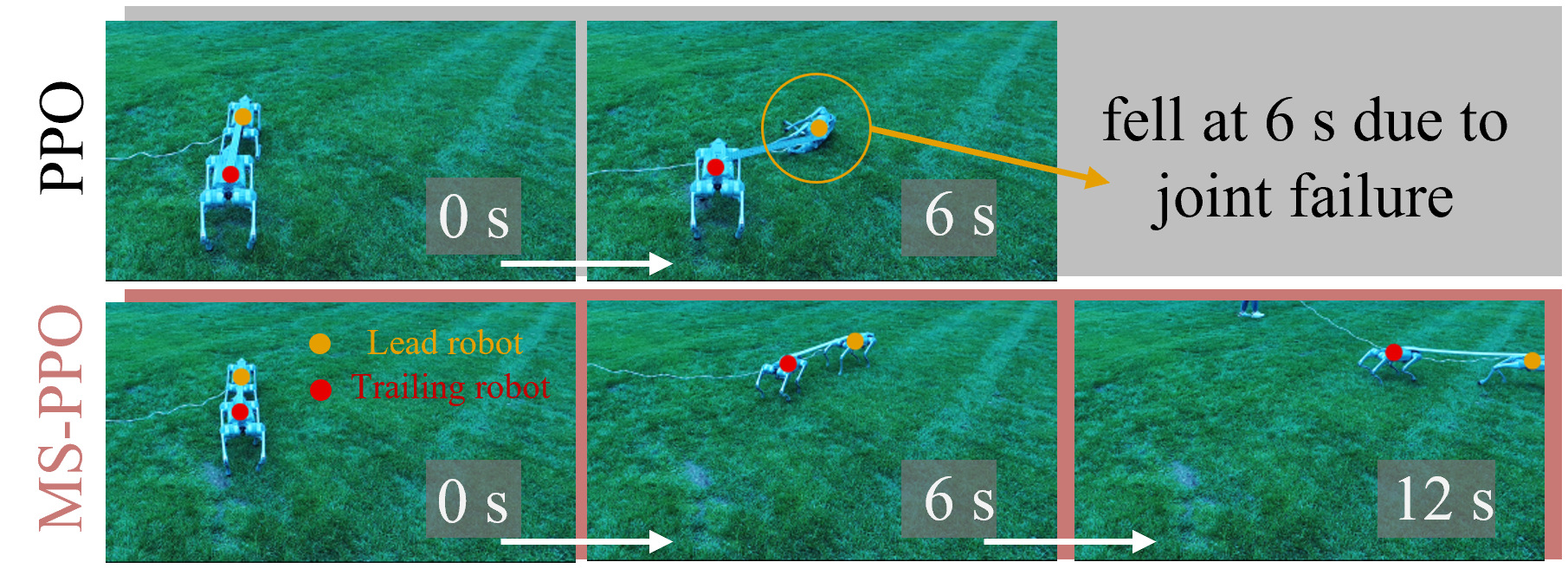}
    \caption{Dual-robot collaboration under joint failure. The lead robot tracks 1\,m/s command with rear-right calf disabled.}
    \label{fig:zero-torque-collaboration}
    \vspace{-5mm}
\end{figure}

\subsection{Training Efficiency Across Tasks}
Figure~\ref{fig:training-efficiency} shows training reward curves for Go2 PPO and G1 FastSAC. On Go2 with joint-failure exposure, \method{} reaches higher rewards earlier and maintains the highest reward toward the end of training. The G1 panel illustrates training dynamics under an adaptive penalty curriculum: penalty weights increase as episode length improves, contributing to the temporary decrease in recorded reward. Because curriculum progression depends on policy performance, different methods may experience different penalty weights at the same training step. We therefore use the G1 curves to illustrate the training process, while assessing policy performance through the common simulation and hardware evaluations in Table~\ref{tab:g1-sim-locomotion}.

Table~\ref{tab:training-efficiency} reports Go2 computational cost. The projected 5000-iteration training time is 6.010 hours for \method{}, compared with 3.005 hours for PPO and 10.814 hours for BoT. Structured parameter sharing yields a compact model, while the fixed robot topology enables message passing through precomputed adjacency matrices and matrix operations, avoiding dynamic edge indexing and scatter-based aggregation. The evaluated BoT implementation instead uses dense attention with topology masks; masking restricts information flow but does not by itself provide sparse computation. Consistent with this computational distinction, \method{} uses fewer inference FLOPs than BoT (4.436 versus $10.854\times10^6$). At batch size one on an RTX 4090, \method{} takes $0.2002$ ms per inference, compared with $0.6087$ ms for BoT, and has the lowest measured peak memory at $1.823$ MiB. Its inference time is approximately 1\% of the 20-ms control period, indicating a small policy-computation budget on the measured GPU. These results show a trade-off: higher training cost than MLP baselines, but improved Go2 learning progress and lower inference cost than the evaluated graph and Transformer baselines.

\begin{table}[t]
    \centering
    \vspace{1mm}
    \caption{
        Go2 training and inference computational cost. Training time is projected to 5000 iterations on an L40S; batch-one inference is measured on an RTX 4090.
    }
    \vspace{-2mm}
    \label{tab:training-efficiency}
    \resizebox{\columnwidth}{!}{%
    \centering
    \begin{tabular}{@{}lcccc@{}}
        \toprule
Method & \begin{tabular}[c]{@{}c@{}}Projected 5k-Iter \\ Training (h)\end{tabular} & \begin{tabular}[c]{@{}c@{}}Inference \\ Latency (ms)\end{tabular} & \begin{tabular}[c]{@{}c@{}}Inference \\ FLOPs ($10^6$)\end{tabular} & \begin{tabular}[c]{@{}c@{}}Peak Inference \\ Memory (MiB)\end{tabular} \\ \midrule
PPO & \textbf{3.005} & $\mathbf{0.0348} \pm \mathbf{0.0012}$ & $\mathbf{2.113}$ & $9.075$                                                                \\
PPO-Aug & 3.363 & $0.0350 \pm 0.0023$ & $\mathbf{2.113}$ & $9.075$                                                                \\
PPO-EMLP & 3.484 & $0.0467 \pm 0.0026$ & $\mathbf{2.113}$ & $41.864$                                                               \\
BoT & 10.814 & $0.6087 \pm 0.0209$ & $10.854$ & $5.886$ \\
MI-PPO & 7.751 & $0.2915 \pm 0.0098$ & $7.685$ & $2.306$                                                                \\
MI-PPO-Aug & 12.101 & $0.2899 \pm 0.0100$ & $7.685$ & $2.306$                                                                \\
\method & 6.010 & $0.2002 \pm 0.0065$ & $4.436$ & $\mathbf{1.823}$                                                                \\ \bottomrule
        \end{tabular}
    }
    \vspace{-10pt}
\end{table}

\begin{figure}[tb]
    \centering
    \newcommand{\figheight}{2.5cm}
    \includegraphics[valign=c,width=0.35\linewidth]{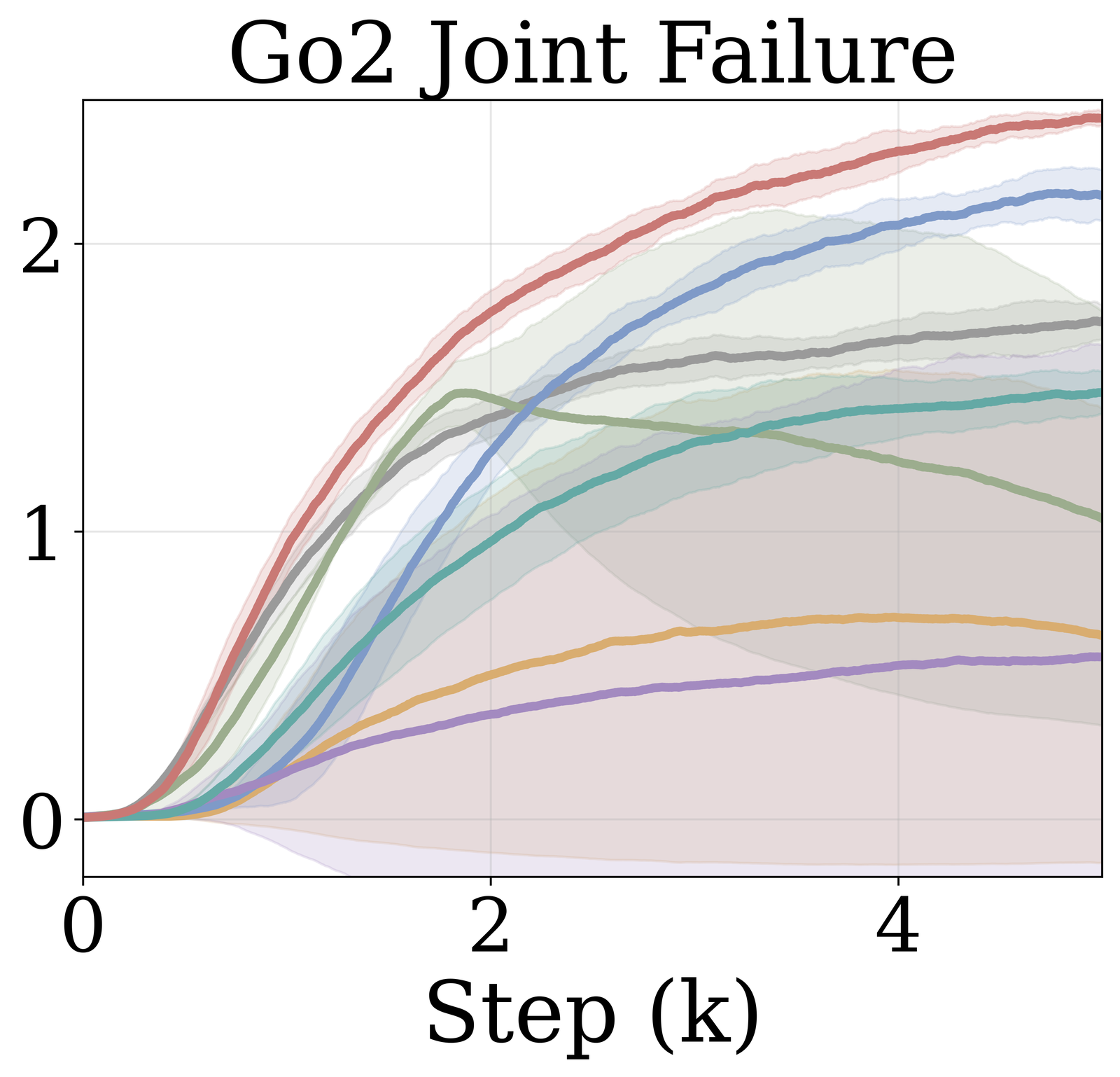}\hfill
    \includegraphics[valign=c,width=0.35\linewidth]{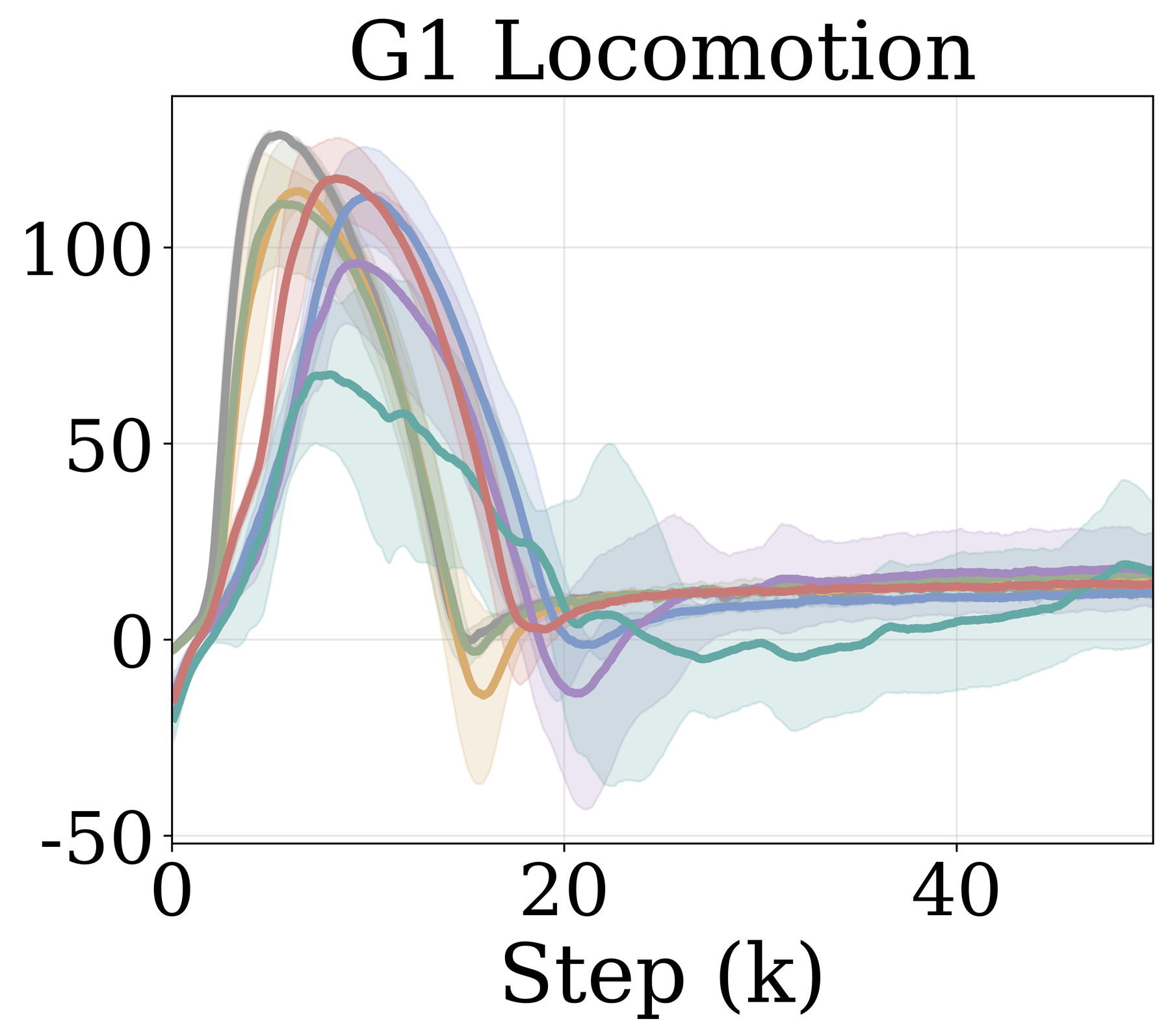}\hfill
    \includegraphics[valign=c,width=0.25\linewidth]{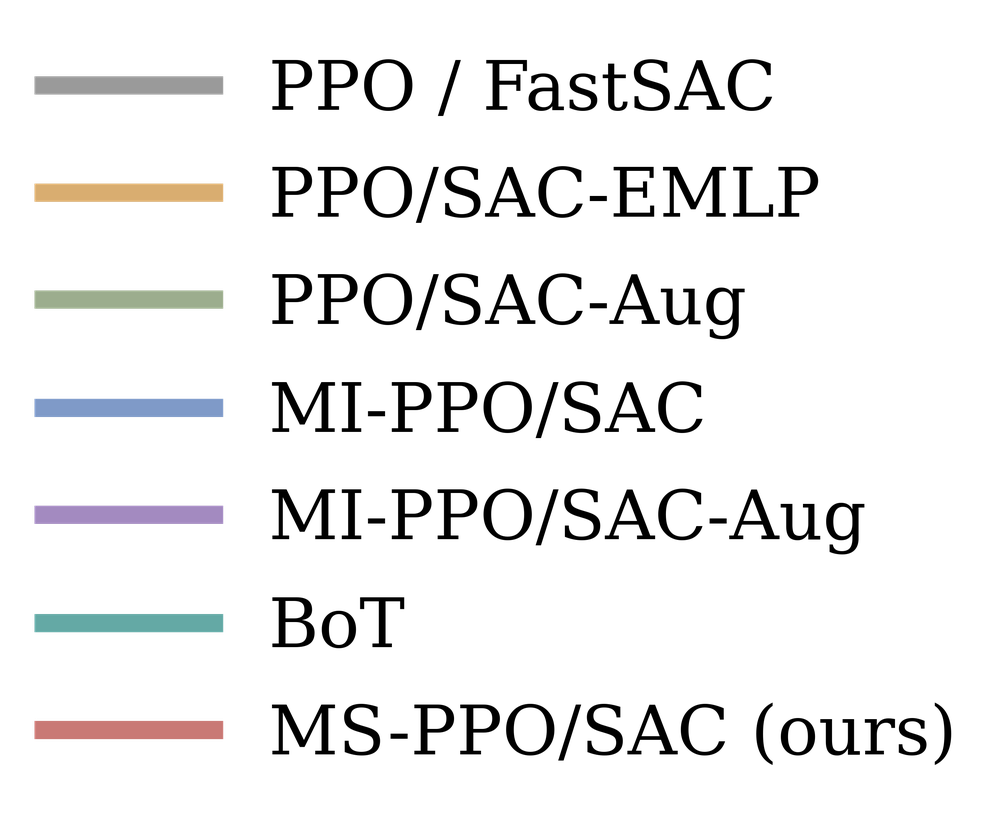}
    \vspace{-2mm}
    \caption{Training reward curves for Go2 locomotion with joint-failure exposure (Sec.~\ref{sec:exp-joing-failure}) and G1 humanoid locomotion (Sec.~\ref{sec:exp-humanoid}). Curves show means and shading shows standard deviations across seeds. G1 curves illustrate training dynamics under an adaptive penalty curriculum; increasing penalty weights contribute to the temporary reward decrease, so raw rewards do not directly rank learning speed across methods.}
    \label{fig:training-efficiency}
    \vspace{-2mm}
\end{figure}

\subsection{Ablation}
We examine the contributions of the equivariant MS-GNN actor and invariant MS-GNN critic by independently replacing each with an MLP (Table~\ref{tab:actor_critic_ablation}). Each configuration is evaluated over all 12 zero-torque joint-failure conditions using three training seeds. Relative to the MLP actor-critic baseline, replacing the MLP actor or critic alone increases mean reward from 5.561 to 8.146 (46.5\%) or 6.764 (21.6\%), respectively. Both replacements reduce tracking errors, with the actor yielding the larger performance gain. Combining the equivariant actor with the invariant critic achieves the highest mean reward and lowest mean linear tracking error. Compared with the equivariant actor with an MLP critic, the invariant critic raises reward to 8.522 (+4.6\%) and reduces linear error to 0.278 (-8.6\%), while angular error increases to 0.390 (+17.1\%). Thus, the invariant critic provides an additional improvement in aggregate return and linear tracking, with a trade-off in angular tracking. Overall, the results support the combined actor-critic design, with the equivariant actor providing the larger gain and the invariant critic further improving the aggregate objective.

\begin{table}[t]
    \centering
    \caption{Actor-critic architecture ablation under zero-torque joint
    failures. Each seed is first averaged over all 12 fault conditions, and
    results are reported as mean $\pm$ standard deviation across three random seeds.  %
    }
    \label{tab:actor_critic_ablation}
    \resizebox{\linewidth}{!}{%
    \begin{tabular}{llccc}
        \toprule
        Actor & Critic
        & Lin. TE $\downarrow$
        & Ang. TE $\downarrow$
        & Reward $\uparrow$ \\
        \midrule
        MLP
        & MLP
        & $0.439 \pm 0.023$
        & $0.567 \pm 0.071$
        & $5.561 \pm 0.317$ \\

        MLP
        & Invariant MS-GNN
        & $0.368 \pm 0.029$
        & $0.463 \pm 0.053$
        & $6.764 \pm 0.559$ \\

        Equivariant MS-GNN
        & MLP
        & $0.304 \pm 0.038$
        & $\mathbf{0.333 \pm 0.037}$
        & $8.146 \pm 0.413$ \\

        Equivariant MS-GNN
        & Invariant MS-GNN
        & $\mathbf{0.278 \pm 0.036}$
        & $0.390 \pm 0.080$
        & $\mathbf{8.522 \pm 0.448}$ \\
        \bottomrule
    \end{tabular}%
    }
    \vspace{-10pt}
\end{table}

\section{Conclusion}

We introduced a morphological symmetry graph representation for locomotion policy learning. By combining kinematic connectivity with morphology-induced permutations and sign changes, the graph specifies both \emph{where} information flows and \emph{how} physical quantities transform across symmetric body parts. The construction yields a graph actor with an equivariant action mean and an invariant critic, instantiated as \method{} on Go2 and \texttt{MS-SAC} on G1.

Experiments demonstrate accurate velocity tracking, generalization to unseen mirrored commands, and single-joint fault compensation on Go2, together with effective 29-DoF whole-body locomotion and zero-shot sim-to-real tracking on G1. On Go2, \method{} combines a smaller parameter count than the evaluated baselines with lower inference cost than the graph and Transformer baselines. Actor-critic ablations show that the equivariant graph actor provides the larger performance gain, while the invariant critic further improves return and linear tracking. Together, these results support jointly encoding kinematic topology and morphological symmetry for locomotion learning across robot embodiments and actor-critic frameworks.

\section{Limitations}

Our validation covers bilateral $\mathbf{C}_2$ symmetry on one quadruped and one humanoid. Graphs, symmetry orbits, and coordinate transformations are specified manually for each morphology; richer symmetry groups and approximate symmetries remain unexplored. Policies are trained separately for each robot, so the results demonstrate applicability of the representation across embodiments rather than direct policy transfer. Although G1 uses a 29-DoF whole-body action space, evaluation focuses on velocity tracking, with hardware tests limited to one forward-velocity command. Broader command distributions and more dynamic whole-body skills require further validation. Fault experiments focus on single-joint zero-torque failures and do not establish robustness to simultaneous faults or other failure modes. Future work will automate symmetry graph construction, extend the representation to approximate symmetries, and evaluate direct cross-embodiment policy transfer and broader hardware operating conditions.

\begingroup
\renewcommand{\baselinestretch}{0.98}
\bibliographystyle{bib/IEEEtran}
\bibliography{bib/IEEEabrv, bib/strings-abrv, reference}
\endgroup

\end{document}